\documentclass{article}

\usepackage[preprint, nonatbib]{neurips_2022}

\usepackage{enumitem} \setlist[itemize]{leftmargin=6mm}
\usepackage[utf8]{inputenc} 
\usepackage[T1]{fontenc}    
\usepackage{hyperref}       
\usepackage{url}            
\usepackage{booktabs}       
\usepackage{amsfonts}       
\usepackage{nicefrac}       
\usepackage{microtype}      
\usepackage{xcolor}         
\usepackage{hyperref}

\usepackage{graphicx,subfigure}

\usepackage[ruled,boxed]{algorithm}
\usepackage{algorithmic}

\usepackage{url,color,verbatim}

\usepackage{amsmath}
\usepackage{amssymb}
\usepackage{mathtools}
\usepackage{amsthm}
\usepackage{setspace}

\usepackage{multirow}
\usepackage{multicol}
\usepackage{geometry}

\usepackage[capitalize,noabbrev]{cleveref}

\theoremstyle{plain}
\newtheorem{theorem}{Theorem}[section]

\newtheorem{lemma}[theorem]{Lemma}

\theoremstyle{definition}

\theoremstyle{remark}
\newtheorem{remark}[theorem]{Remark}

\usepackage[textsize=tiny]{todonotes}

\def\calE{\mathcal{E}}
\def\calR{\mathcal{R}}

\def\Gbar{ {\overline{G}} }

\def\Abar{ {\overline{A}} }
\def\bbar{ {\overline{b}} }
\def\LGbar{ {\mathcal{L}_{\overline{G}}}  }

\def\rhoT{\overline{\rho}}
\def\mH{\mathcal{H}}
\def\supp{\mathrm{supp}}

\def\R{\mathbb{R}}

\newcommand{\innerp}[1]{\langle{#1}\rangle}
\newcommand{\dbinnerp}[1]{\langle\hspace{-1mm}\langle{#1}\rangle\hspace{-1mm}\rangle}

\newcommand{\floor}[1]{\lfloor{#1}\rfloor}

\newcommand{\argmax}[1]{\underset{#1}{\operatorname{arg}\operatorname{max}}\;}

\newcommand{\e}[1]{{\small $#1$}}

\title{Nonparametric learning of kernels \\ in nonlocal operators}

\author{%
   Fei Lu  \texttt{and} Qingci An\\
   Department of Mathematics\\ Johns Hopkins University\\
  Baltimore, MD, USA \\
  \texttt{feilu@math.jhu.edu, qan2@jhu.edu} 
  \And
  Yue Yu \\
  Department of Mathematics\\ Lehigh University\\ Bethlehem, PA, USA\\
  \texttt{yuy214@lehigh.edu} 
}

\begin{document}

\maketitle

\begin{abstract}
   Nonlocal operators with integral kernels have become a popular tool for designing solution maps between function spaces, due to their efficiency in representing long-range dependence and the attractive feature of being resolution-invariant. In this work, we provide a rigorous identifiability analysis and convergence study for the learning of kernels in nonlocal operators. It is found that the kernel learning is an ill-posed or even ill-defined inverse problem, leading to divergent estimators in the presence of 
   modeling errors or measurement noises. To resolve this issue, we propose a nonparametric regression algorithm with a novel data adaptive RKHS Tikhonov regularization method based on the function space of identifiability. The method yields a noisy-robust convergent estimator of the kernel as the data resolution refines, on both synthetic and real-world datasets. In particular, the method successfully learns a homogenized model for the stress wave propagation in a heterogeneous solid, revealing the unknown governing laws from real-world data at microscale.  Our regularization method outperforms baseline methods in robustness, generalizability and accuracy.
\end{abstract}

\section{Introduction} 

During the last 20 years there has been a lot of progress in the design of machine learning models; however, their employment in scientific discovery of hidden physical laws of complex system is relatively recent. One of the most classical examples has been learning the material constitutive law directly from experimental data, where the governing partial differential equation (PDE) is unknown and the data features noises from measurements. This task 
can be formulated as identifying an operator that continuously maps the displacement field to the loading field. Moreover, such a need for learning operators between function spaces has also become vital in other applications ranging from homogenization problems \cite{you2021_DatadrivenLearning,you2022_DatadrivenPeridynamic,lin2021seamless,lin2021operator}, fast PDE solvers \cite{lu2019deeponet,lu2021learning,li2020neural,li2020multipole,li2020fourier,kovachki2021neural}, to control problems \cite{lu2021learning,hwang2021solving}, just to name a few. 
Among these work, models with nonlocal operators have 
have received increasing attention, since they can describe physical phenomena that classical differential operators fail to capture and provide a powerful model for a large class of complex engineering and scientific applications \cite{kovachki2021neural,tao2018nonlocal,you2020_DatadrivenLearning,burkovska2021optimization}. However, despite a long line of work on nonlocal model learning and its applications, 
there is limited theoretical characterization of the underlying inverse problem, even in the linear setting.

In this paper, we study the learning of kernels in linear nonlocal diffusion operators from data. Suppose that we are given data 
{\small\begin{equation}\label{eq:data}
\mathcal{D}  = \{u_i,f_i\}_{i=1}^N = \{u_i(x_j),f_i(x_j): j=1,\ldots,J\}_{i=1}^N,
\end{equation}}
where $u_i,f_i$ are real-valued continuous functions on a bounded open connected set $\Omega\subset \R^d$ and $\{x_j\in \Omega\}$ are spatial mesh points. The task is to learn the kernel $\phi$ in a nonlocal diffusion operator $L_\phi$ mapping $u_i$ to $f_i$ of the form:
\begin{equation}\label{eq:nonlocal_op}
L_\phi[u](x) =  \int_{\Omega} \phi(|y-x|) [u(y)-u(x)]dy= f(x), \forall x\in \Omega. 
\end{equation} 
 This operator is simple yet flexible for nonlocal models: it has been employed in \cite{you2021_DatadrivenLearning,you2022_DatadrivenPeridynamic} to provide a homogenized material model from microscale measurements, and can be seen as a simplified linear version of the integral neural operator in \cite{you2022nonlocal}.


Our goal is to infer the kernel $\phi$ from data via nonparametric regression, so as to address the general situations that there is limited information to derive a parametric form or constraints for the kernel, which can be either smooth or singular. The regression utilizes the linear dependence of the operator on the kernel, making it possible to treat the large size functional data in a scalable fashion. 

Three challenges are to be overcome. First, the function space of identifiability (FSOI) is yet to be specified properly, otherwise the inverse problem can be ill-defined in the sense that there are multiple kernels fitting the data. This is fundamentally different from classical nonparametric regression that learns a function $Y=\phi(X)$ from random samples $\{(X_i,Y_i)\}$ from the joint distribution of $(X,Y)$, for which the FSOI is $L^2(\rho)$ with $\rho$ being the distribution of $X$ and the optimal estimator is the conditional expectation.   
 Second, the kernel estimator should be resolution independent and converge in a proper function space when the data resolution refines, so that it can be applied to problems and simulation tasks with different grids or discretization methods and provides a guaranteed modeling accuracy. 
 Third, beyond resolution invariance, the estimator should be robust to imperfect data so as to be applicable in real applications. 
 
We introduce a nonparametric learning method producing a convergent estimator for the kernel as the data mesh refines. 
Our main result has two folds. Firstly, we specify a data adaptive function space of learning and provide an identifiability theory, which shows that the function space of identifiability is the closure of a system intrinsic data adaptive reproducing kernel Hilbert space (SIDA-RKHS), which is equipped with an exploration measure indicating the information from data. Beyond this space, the inverse problem may become undetermined. Furthermore, the identifiability theory shows that the inverse problem is ill-posed and it becomes more ill-posed when the data resolution refines, which contradicts the intuition that learning from more datasets and refined mesh grids should provide an improved estimator. 
Thus, regularization is necessary. Second, we introduce a new regularization method that plays a key role in identifying a convergent estimator in the presence of model errors and/or measurement noises. 
It uses the norm of the SIDA-RKHS. In experimental studies, we compare our proposed SIDA-RKHS regularization method with two common Tikhonov/ridge regularizers that use $l^2$ and $L^2$ norms. Results on both benchmark problems with synthetic data and real-world data show that only the SIDA-RKHS regularizer can consistently obtain convergent estimators for all types of kernels, especially when the data is noisy.

We summarize our major contributions below:
\vspace{0.7mm}\newline
1) We establish a rigorous identifiability theory for the  nonparametric learning of kernels in nonlocal operators, and for the first time specifying a system-intrinsic data-adaptive function space of identifiability (see Lemma \ref{lemma:Gbar} and Theorem \ref{thm:space_ID}).
The theory also indicates a pitfall of the nonlocal kernel learning problem: the inverse problem is ill-posed. 
\vspace{0.7mm}\newline
2) We introduce a nonparametric regression algorithm equipped with a novel regularization method based on the SIDA-RKHS (see Section \ref{sec:alg}), which overcomes the ill-posedness to yield a convergent estimator robust to noise. 
\vspace{0.7mm}\newline
3) We validate the theory and the proposed algorithm on a number of benchmark problems, including various synthetic datasets and a real-world dataset where the governing law is unknown (see Section \ref{sec:num}). Results show that the proposed algorithm provides a stable and converging estimator, while the common Tikhonov/ridge regularizers with $l^2$ or $L^2$-norm fail this task.

\vspace{-1mm}
\subsection{Related Work}
\vspace{-0.04in}

\noindent\textbf{Nonlocal operators:} Nonlocal operators arise in various areas such as nonlocal and fractional diffusions \cite{silling2007_PeridynamicStates,du2012_AnalysisApproximation,delia2020_NumericalMethods,applebaum09,andreu-vaillo2010_NonlocalDiffusion,bucur2016_NonlocalDiffusion,chen2017_HeatKernels,xiong2019unique,wang2012fast}, de-noising and regularization by nonlocal kernels \cite{kindermann2005deblurring,gilboa2009nonlocal,holler2020_LearningNonlocal}, multi-agent systems with nonlocal interaction \cite{LZTM19pnas,LMT21,LangLu21} and nonlocal networks \cite{wang2018_NonlocalNeural,li2020_NeuralOperator}. 
The inverse problem for nonlocal diffusions has been studied in \cite{jin2015tutorial,li2021extracting} from a single solution. To discover hidden nonlocal physical laws from data, a parametric nonlocal kernel learning approach has been proposed in \cite{you2020_DatadrivenLearning,you2022_DatadrivenPeridynamic}, where the coefficients of Bernstein polynomials are learnt with physics-based constraints and a Tikhonov regularization. Beyond the linear nonlocal model and the nonlocal kernel regression methods, nonlocal operators were further combined with neural networks, and nonlocal kernel networks were developed for learning maps between high-dimensional variables in dynamical systems \cite{li2020fourier,li2020multipole} or function spaces \cite{li2020neural,you2022nonlocal}. An attractive feature of these nonlocal kernel/operator learning methods is the generalizability among approximations corresponding to different underlying levels of resolution and discretization. 
However, as seen in \cite{you2021_DatadrivenLearning,you2022_DatadrivenPeridynamic,li2020neural,you2022nonlocal}, neither the nonlocal kernel learning methods nor nonlocal kernel networks yield estimator convergence when trained on finer resolution, and the test error may even increase. This fact indicates the possible ill-posedness of the learning problem. In this work, we tackle this issue {by introducing a new regularization method based on a data-adaptive RKHS in a nonparametric learning approach.}
 

\noindent\textbf{Functional data analysis:} Functional data analysis (see e.g., \cite{hsing2015theoretical,kadri2016operator,ferraty2006nonparametric} and the references therein) studies the learning an infinite-dimensional operator from functional data. In contrast, we focus on learning a radial kernel in an operator, exploiting the low-dimensional structure of the operator, which enables us to learn the kernel (hence the operator) from limited data.

\noindent\textbf{Regularization methods:} Our SIDA-RKHS regularization is a type of Tikhonov/ridge regularization that adds a penalty term to the loss function. It differs from previous methods at the penalty term. The commonly used penalty terms include the Euclidean norm in the classical Tikhonov regularization \cite{hansen1994_regularization_tools,hansen_LcurveIts_a}, the RKHS norm with an ad hoc reproducing kernel (often the Gaussian kernel) \cite{CS02,Cucker2007}, the total variation norm in the Rudin-Osher-Fatemi method \cite{rudin1992nonlinear}, or the $L^1$ norm in LASSO \cite{tibshirani1996_RegressionShrinkage}. Whereas each of these penalty terms has their specific applications, none of them take into account of the FSOI, which is fundamental for learning kernels in operators. 
Also, our regularization method is inspired by the kernel flow method that learns hyper-parameters of the reproducing kernel  \cite{owhadi2019_KernelFlows,hamzi2021_LearningDynamical,chen2021_ConsistencyEmpirical}, but our reproducing kernel is determined by the system and the data. Given the importance of regularization to overcome ill-posedness and overfitting, we expect our SIDA-RKHS regularization method to be applicable to a wide range of linear inverse problems and machine learning methods.

 \vspace{-2mm}
\section{Learning theory and algorithm}\label{sec:method}

\vspace{-1mm}
\subsection{Nonparametric regression with regularization}\label{sec:nonparaReg}\vspace{-1mm}
We construct an estimator by minimizing the loss functional of mean square error: 
\vspace{-0.03in}\begin{equation}\label{eq:lossFn}
\calE(\phi) = \frac{1}{N}\sum_{i=1}^N \|L_\phi[u_i]-f_i\|_{L^2}^2.
\end{equation} 
Here we consider only the $L^2$ norm is \e{\|f\|_{L^2}^2 = \int |f(x)|^2dx}, which has minimal requirements on the data. 
Other norms (e.g., the Sobolev norms) 
can also be used when the data are smooth.

Note that the loss functional is quadratic in $\phi$ because the nonlocal operator is linear in $\phi$. Thus, the minimizer of the loss functional is the least squares estimator (LSE), which is handy once one selects a set of basis functions for a hypothesis space. More specifically, suppose the hypothesis space is \e{\mH_n =\mathrm{span}\{\phi_i\}_{i=1}^n}, 
for each \e{\phi =\sum_{i=1}^n c_i \phi_i\in \mH_n}, 
 we can write the loss functional in \eqref{eq:lossFn} as 
\e{\calE(c) =  \calE(\phi) = c^\top \Abar_n c - 2c^\top \bbar_n+ C_N^f}, 
where \e{C_N^f = \frac{1}{N}\sum_{k=1}^N  \int |f_i (x) |^2 dx}, and the normal matrix $\Abar$ and vector $\bbar$ are given by
\vspace{-0.03in}{ \begin{equation}\label{eq:Ab} 
\Abar_n(i,j) = \dbinnerp{\phi_i,\phi_j}, \, \bbar_n(i)= \frac{1}{N}\sum_{k=1}^N  \int  L_{\phi_i}[u_k](x) f_k(x)dx,
\end{equation}}\vspace{-0.11in}
and the bilinear form $\dbinnerp{\cdot,\cdot}$ is defined by 
{ \begin{equation} \label{eq:binlinearForm}  
{\dbinnerp{\phi,\psi}= \frac{1}{N}\sum_{k=1}^N \int_{\R^d}   L_{\phi}[u_k](x) L_{\psi}[u_k](x) dx.}  
\end{equation}}
The least squares estimator is computed directly from the minimizer of the quadratic function $\calE(c) $:  
\vspace{-0.03in}\begin{equation}\label{eq:LSE} 
 \widehat \phi_{\mH_n}= \sum_i  \widehat c_i \phi_i, \quad \text{ where  }  \widehat c = \Abar_n^{-1}\bbar_n, 
\end{equation}
where \e{\Abar_n^{-1}} is the inverse (or pseudo-inverse when the inverse does not exist) of \e{\Abar_n}. 

However, the above least squares regression encounters a big challenge in obtaining convergent estimators for this ill-posed inverse problem (see Section \ref{sec:Identifiability}). As a nonparametric method, it is often necessary to select a relatively large hypothesis space to make the model flexible enough. However, 
the large hypothesis space leads to a normal matrix that is often severely ill-conditioned. As a result, the estimator in \eqref{eq:LSE}  oscillates violently when the data is imperfect due to either measurement noise or model error, and the estimator does not converge when the data mesh refines. 

Regularization methods overcome the ill-posedness by adding a penalty term to the loss functional:   
\begin{align}\label{eq:lossFn_reg}  
\calE_\lambda(\phi) = 
\calE(\phi) +\lambda \calR(\phi),
\end{align}
where $\calR(\phi)$ is a regularization term, and $\lambda$ is a hyper-parameter controlling the contribution of the regularization term. 
Various penalty terms have been proposed, 
however, none of them take into account of the function space of identifiability, which is at the foundation of learning (see Section \ref{sec:Identifiability}).  Based on it, we will introduce a data-adaptive RKHS regularization method (in Section \ref{sec:alg}). Thus, it is different from classical regularization using an ad hoc RKHS \cite{Cucker2007,bauer2007regularization}.

\vspace{-1mm}
\subsection{Function space of identifiability}\label{sec:Identifiability}\vspace{-1mm}
The identifiability theory characterizes the function space of learning. There are two key elements in our identifiability theory: 1) an exploration measure, which is a probability measure that quantifies the exploration of the kernel's variable by the data, and 2) the function space of identifiability, in which the loss functional has a unique minimizer. They are described as follows.


\textbf{The exploration measure.} As the first key element, we introduce first a novel measure on $\R_+$ that quantifies the exploration of the independent variable of the kernel by the data. We assume the radial kernel's support to be in an interval $[0,R_0]$. A given dataset may only explore part of this interval. More specifically, the discrete data set in \eqref{eq:data} explores only the pairwise distances \e{|x_j-x_k|} in \e{\mathcal{R}_N^J  =\{r_{ijk} = |x_j-x_k| \leq R_0:  u_i(x_j)-u_i(x_k)\neq 0 \text{ for some } i,j,k\}}, the set of all the pairwise distances \e{|x_j-x_k|} with repetition. We define an empirical measure and its continuous limit 
\begin{equation}
\begin{aligned}\label{eq:rho_disc}
\rho_N^J(dr) & =  \frac{1}{|\mathcal{R}_N^J|}\sum_{i=1}^N\sum_{ j,k=1}^J \delta_{|x_j-x_k|}( r) w_i(x_j,x_k), \\ 
\rho_N(dr)& =  \frac{1}{ZN}\sum_{i=1}^N \int_\Omega\int_\Omega \delta_{|x-y|}(r)  w_i(x,y) dxdy
\end{aligned}
\end{equation}
for $r\in [0,R_0]$, where $|\mathcal{R}_N^J|$ is the cardinality of the set $\mathcal{R}_N^J$, $\delta_s(r)$ is the Dirac distribution with point mass at $s$, and $Z$ is the normalizing constant. Here the weight function is  \e{w_i(x,y)=|u_i(x)-u_i(y)|}.

The exploration measure plays an important role in the learning of the kernel.  It reflects the strength of exploration to $|x-y|$ by the data \e{|u_i(x)-u_i(y)|} in the loss function, and it will act as a re-weighting factor through the SIDA-RKHS regularization to be introduced in Section \ref{sec:alg}. Thus, we will use it to quantify the accuracy of the kernel's estimator in $L^2(\rho_N)$ (or $L^2(\rho_N^J)$ for discrete data). 

\noindent\textbf{Main result: function space of identifiability.}  
We define the function space of identifiability (FSOI) as the largest linear space in which the loss functional has a unique minimizer. In other words, the variational inverse problem of finding a unique minimizer of the loss functional is well-defined in this space.  
In the following, we write only the continuous function space $L^2(\rho_N)$, but all the arguments apply to the discrete function space $L^2(\rho_N^J)$ in an obvious manner (see Remark \ref{rmk:discreteFrechet}).

\begin{theorem}[Function space of identifiability]\label{thm:space_ID}
 Consider the problem of learning the kernel $\phi$ by minimizing the loss functional $\calE$ in \eqref{eq:lossFn} with $\{u_i, f_i\}_{i=1}^N$ being continuous in a bounded domain $\Omega$. Then, the function space of identifiability (FSOI), the largest subspace of $L^2(\rho_N)$ in which $\calE$ has a unique minimizer, is the eigen-space of nonzero eigenvalues of $\LGbar$, an integral operator defined by 
 \begin{equation}\label{eq:LG} 
\LGbar \phi (r) =  \int_0^\infty  \phi(s) \overline{G}(r,s) \rho_N(ds).   
\end{equation}
Here the integral kernel $\Gbar$ comes from data: 
\begin{equation}\label{eq:Gbar} 
   \Gbar(r,s) = [\rho_N'(r)\rho_N'(s)]^{-1}G(r,s), 
\end{equation} 
where $\rho_N'$ is the density of $\rho_N$ and $G$ is 
 \begin{align} \label{eq:G}
 G(r,s)= \frac{1}{N}\sum_{i=1}^N&\int_{|\eta | =1} \int_{|\xi |=1}  \left[ \int  [u_i(x+r\xi) - u_i(x) ] \right. 
\left.[u_i(x+s\eta) - u_i(x) ] dx \right]d\xi d\eta,
\end{align}
for $ r,s\in \mathrm{supp}(\rho_N)$, and $G(r,s) =0$ otherwise. Furthermore, the minimizer of $\calE$ is 
 \begin{equation*} 
\widehat \phi  = \LGbar^{-1} P\phi_N^f,  
\end{equation*}
where $P$ is the projection to the FSOI. Here \e{\phi_N^f \in L^2(\rho_N)} is the Riesz representation of the bounded linear functional defined by 
\e{\innerp{\phi_N^f,\psi}_{L^2(\rho_N)} = \frac{1}{N}\sum_{i=1}^N \int 2 L_\psi[u_i](x) f_i(x)dx, \, \forall \psi\in L^2(\rho_N).}
\end{theorem}

When the data is continuous and noiseless, the true kernel is the unique minimizer, i.e., it is identifiable by the loss functional, if it is in the FSOI. Also, note that the FSOI is data-dependent. When the data is discrete or noisy, the unique minimizer is an optimal estimator in the FSOI, and it converges to the true kernel as the data mesh refines (see Remark \ref{rmk:discreteFrechet}) and the noise to signal ratio reduces.

The proof of Theorem \ref{thm:space_ID} is based on the uniqueness of zero of the Fr\'echet derivative of the loss functional, which becomes clear from the following lemma. Their proofs are deferred to Appendix \ref{sec:proofs}. 

\begin{lemma}[The Fr\'echet derivative of the loss functional]\label{lemma:D_lossFn}
The Fr\'echet derivative of the loss functional $\calE$ in $L^2(\rho_N)$, with  $\LGbar$ defined in \eqref{eq:LG} and $\phi_N^f$ defined in Theorem {\rm\ref{thm:space_ID}}, is   
\begin{equation*}
\nabla \calE(\phi) =  2 ( \LGbar \phi - \phi_N^f). 
\end{equation*} 
\end{lemma} 

\textbf{System-intrinsic data-adaptive RKHS.} Theorem \ref{thm:space_ID} highlights two fundamental challenges: the inverse problem is well-defined only in the FSOI, and it is ill-posed in the FSOI because it involves the inverse of a compact operator $\LGbar$ (as shown in the next lemma). Fortunately, the integral kernel $\Gbar$ defines a reproducing kernel Hilbert space (RKHS), which provides a regularization norm to ensure the learning to take place in the FSOI and to overcome the ill-posedness. This RKHS is system intrinsic as it depends on the structure of the system of nonlocal operators, and it is data-adaptive,  utilizing both the exploration measure and the data $\{u_i\}$. Thus, we call it SIDA-RKHS. 


\begin{lemma}[Characterization of the SIDA-RKHS] \label{lemma:Gbar}
Suppose that  the data $\{u_i\}$ are continuous in $\Omega$. Then, the following statements hold true. \vspace{-3mm}
\begin{itemize}\setlength\itemsep{-0.5mm}
\item[(a)]  The integral kernel $\Gbar$ defined in \eqref{eq:Gbar}  is positive semi-definite. 
\item[(b)] The integral operator $\LGbar: L^2(\rho_N)\to L^2(\rho_N)$  defined in \eqref{eq:LG}
is compact and positive semi-definite, and we have, for any $\phi,\psi\in L^2(\rho_N)$,
\begin{equation}\label{eq:dbinnerp_LG}
\dbinnerp{\phi,\psi} = \innerp{\LGbar \phi,\psi}_{L^2(\rho_N)}.
\end{equation}
\item[(c)] The RKHS $H_G$ with $\Gbar$ as reproducing kernel satisfies $H_G = \LGbar^{1/2} (L^2(\rho_N) )$, and its inner product satisfies  $\innerp{\phi,\psi}_{H_G} = \innerp{\LGbar^{-1/2}\phi,\LGbar^{-1/2}\psi}_{L^2(\rho_N)}$ for any $\phi,\psi\in H_G$.
\item[(d)] The eigenvalues of $\LGbar$ converges to zero, and its eigen-functions $\{\psi_k\}_{k}$ form a complete orthonormal basis of $L^2(\rho_N)$. For any $\phi= \sum_k c_k \psi_k$, we have 
\begin{equation}\label{eq:norms}
\begin{aligned}
\dbinnerp{\phi,\phi} = \sum_k \lambda_k c_k^2, \quad \|\phi\|^2_{L^2(\rho_N)} = \sum_k  c_k^2, \quad \|\phi\|^2_{H_G} = \sum_k \lambda_k^{-1} c_k^2, 
\end{aligned}
\end{equation}\vspace{-1mm}
where the last equation is restricted to $\phi\in H_G$. 
\end{itemize}
\end{lemma}


\begin{algorithm}[h!]
{\small
\caption{Nonparametric learning of the nonlocal kernel with SIDA-RKHS regularization}\label{alg:main}
\begin{algorithmic}
\STATE {\textbf{Input:} The data $\{u_i,f_i\}_{i=1}^N = \{u_i(x_j),f_i(x_j)\}_{i,j=1}^{N,J}$  to construct the nonlocal model $L_\phi[ u]=f$.}
\STATE {\textbf{Output:} Estimator $\widehat \phi$}
\STATE {1. Estimate the exploration measure $\rho_N^J$ as in \eqref{eq:rho_disc}, and denote $R$ the upper bound of its support.}
\STATE {2. Get regression data (see Appendix \ref{sec:append_alg}). 
} 
\STATE {3. Select a class of hypothesis spaces $\mH_n =\mathrm{span}\{\phi_i\}_{i=1}^n$ with $n$ in a proper range. 
}
\STATE {4. \textbf{For} {$n$ in the range}}
\STATE {\quad 4a) Compute $(\Abar_n, \bbar_n, B_n)$ for $\mH_n =\mathrm{span}\{\phi_i\}_{i=1}^n$ with $B_n = ( \innerp{\phi_i,\phi_j}_{L^2(\rho_N^J)})_{1\leq i,j\leq n}$;
}
\STATE {\quad 4b) If the basis matrix $B_n$ is singular, stop and remove $n$ from the range;}
\STATE {\quad 4c) Solve the generalized eigenvalue problem $\Abar_n V =  B_n \Lambda V$, where $\Lambda$ is the diagonal matrix of eigenvalues and $V^\top B_n V= I_n$;}
\STATE {\quad 4d) Compute the RKHS-norm matrix $B_{rkhs} = (V\Lambda V^\top)^{-1}$; 
}
\STATE {\quad 4e) Use the L-curve method find an optimal estimator $\widehat \phi_{\lambda_n^*}$.
}
\STATE {5. Select the optimal dimension $n^*$ (and degree if using B-spline basis) that has the minimal loss value (along with other cross-validation criteria if available). Return the estimator $\widehat \phi = \sum_{i = 1}^{n^*} c^i_{n^*} \phi_i$.}
\end{algorithmic}
}\vspace{-1mm}
\end{algorithm}

\begin{remark}[Discrete data]\label{rmk:discreteFrechet}
When the space $L^2(\rho_N^J)$ is a discrete vector space due to discrete data, we learn the kernel on finitely many points $\{r_k\}_{k=1}^n$ explored by the data. In this case, the integral kernel $G$ in {\rm\eqref{eq:G}} becomes a positive semi-definite matrix in $\R^{n\times n}$, so is $\Gbar$ in {\rm\eqref{eq:Gbar}}. Now the operator $\LGbar$ is defined by the matrix $\Gbar$ on the weighted vector space $\R^n$ and its eigenvalues is the generalized eigenvalue of $(\Gbar,B_n)$ with $B_n$ being the diagonal matrix of $\rho_N^J$. As a result, the SIDA-RKHS $H_G$ is the vector space spanned by the eigenvectors with nonzero eigenvalues. Furthermore, its norm in {\rm \eqref{eq:norms}} can be computed directly from the eigen-decomposition. This norm is better suited for regularization even when the SIDA-RKHS has the same dimension as \e{L^2(\rho_N^J)} (or dense in it). As data mesh refines, these vector spaces converges to the corresponding function spaces when the data is smooth. 
\end{remark}\vspace{-1mm}

\vspace{-2mm}
\subsection{Algorithm: LSE with SIDA-RKHS regularization}\label{sec:alg} \vspace{-1mm}
Based on the function space of identifiability, we introduce next a nonparametric learning algorithm with SIDA-RKHS regularization. 
The algorithm consists of three steps. 
First, we utilize the data to estimate the exploration measure and the support of the kernel. Based on them, we set a class of hypothesis spaces, with their dimensions,  i.e., the number of basis functions, in a proper range moving from under-fitting to over-fitting. For the hypothesis space $\mH_n =\mathrm{span}\{\phi_i\}_{i=1}^n$, we compute the basis matrix $B_n = ( \innerp{\phi_i,\phi_j}_{L^2(\rho_N^J)})_{1\leq i,j\leq n} \in \R^{n\times n}$.
Second, we assemble the regression matrices 
from data for each of these hypothesis spaces. We approximate the integrals by Riemann sum or other numerical integrator. 
Finally, we identify an estimator with SIDA-RKHS regularization for each of these hypothesis spaces  by the L-curve method \cite{hansen_LcurveIts_a} and select the one with the best fitting. 
We summarize the method in Algorithm \ref{alg:main}, with its full details provided in Section \ref{sec:append_alg}.

The core innovations are the exploration measure and the regularization using the SIDA-RKHS norm. Importantly, they bring little extra computational cost. The exploration measure is available directly from data. The SIDA-RKHS norm is computed directly from the triplet $(\Abar_n, \bbar_n, B_n)$ using the generalized eigenvalue problem as detailed in the algorithm.

Our SIDA-RKHS regularization uses the RKHS norm $\calR(\phi) = c^\top B_{rkhs} c$, where $B_{rkhs}$ is defined in (4d) in Algorithm \ref{alg:main}. It differs from the commonly-used Tikhonov/ridge regularization using either the $l^2$-norm that sets $\calR(\phi) = \sum_i c_i^2$ or the $L^2(\rho_N^J)$-norm that sets $\calR(\phi) = c^\top B_n c$. We note that the three norms become the same when $B_n=I_n$ and all the eigenvalue of $A_n$ are $1$. 

\vspace{-3mm}
\section{Tests on synthetic and real-world data}\label{sec:num}\vspace{-2mm}
We test our nonparametric learning method on both synthetic data and real-world data in 1D examples. On each dataset, we compare our SIDA-RKHS regularizer with two baseline regularizers using the $l^2$ and $L^2$ norm (denoted as l2 and L2, respectively). All these regularizers use the same L-curve method to select the hyper-parameter $\lambda$ as described in Appendix \ref{sec:append_Lcurve}. In the case of synthetic data, we systematically examine the method with three types of kernels in the regimes of noiseless and noisy data. Since the ground-truth kernel is known, we study the convergence of estimators to the true kernel as the data mesh refines. We also apply our method to a real-world dataset for stress wave propagation in a heterogeneous bar, with the goal of constructing a homogenized model from microscale data.  Since there is no ground-truth, 
we examine the performance of estimators by studying their physical stability and capability of reproducing the wave motion on a cross-validation dataset. 
All datasets and codes used will be publicly released on GitHub. 

\textbf{Settings for the learning algorithm.} In implementation of Algorithm \ref{alg:main}, we use B-spline basis functions consisting of piece-wise polynomials with degree 2 so that the estimated kernel is twice differentiable (see Section \ref{sec:Bsplines} for a brief introduction of B-splines). 
The knots of B-splines are evenly spaced on interval $[0, R]$, with one additional knots at $0$ to make the first basis nonzero at $x=0$.  We select the dimension with minimal loss from a sequence of dimensions in the range $\floor{\frac{R}{\Delta x}} \times [0.2,1]$ 
as long as the basis matrix $B_n$ is well-conditioned.

\vspace{-2mm}
\subsection{Examples with synthetic data} \vspace{-2mm}
\textbf{Numerical settings.} 
We consider three kernels: a sine kernel, a Gaussian kernel, and a fractional Laplacian kernel (specified below). They act on the same set of functions $\{u_i\}_{i=1,2}$ with $u_1= \sin(x)\mathbf{1}_{[-\pi,\pi]}(x)$ and $u_2(x) =\cos(x)\mathbf{1}_{[-\pi,\pi]}(x)$. In the ground-truth model, the integral $L_\phi[u_i]$ is computed by the adaptive Gauss-Kronrod quadrature method, which is much more accurate than the Riemann sum integrator that we will use in the learning stage.
To create discrete datasets with different resolutions, for each $\Delta x\in0.0125\times\{1,2,4,8,16\}$, we take values $\{u_i,f_i\}_{i=1}^N = \{u_i(x_j),f_i(x_j):x_j \in {[-40,40]}, j=1,\ldots,J\}_{i=1}^N$, where $x_j$ is a point on the uniform grid with mesh size $\Delta x$.

For each kernel, we consider both noiseless and noisy data with different noise levels, with a noise-to-signal-ratio ($nsr$) taking values $\{0,0.5,1,2\}$. Here the noise is added to each spatial mesh point, independent and identically distributed centered Gaussian with standard deviation $\sigma$, and the noise-to-signal-ratio is the ratio between $\sigma$ and the average $L^2$ norm of $f_i$.  

The three ground-truth kernels are specified as follows.\newline\vspace{0.5mm}
\noindent$\bullet$  \emph{Sine kernel.} The sine kernel is $\phi_{true}(r)= \sin(6r)\mathbf{1}_{[0,10]}(r)$. This sine kernel represent a smooth oscillating kernel in the same class as the data $u_i$. The estimated support is in $[0,R]$ with $R=11.02$.  \newline \vspace{0.5mm}
\noindent$\bullet$  \emph{Gaussian kernel.} The Gaussian kernel $\phi_{true}$ is the Gaussian density centered at 5 with standard deviation 1. This kernel represents a smooth kernel. 
It has $R=11.58$. 
\newline \vspace{0.5mm}
\noindent$\bullet$ \emph{Fractional Laplacian kernel.} It is a truncated version of the fractional Laplacian kernel 
that has been widely studied in fractional and nonlocal diffusions (see e.g., \cite{bucur2016_NonlocalDiffusion,applebaum09,du2012_AnalysisApproximation,wang2012fast}). We set 
$\phi_{true}(r) = c_{d,s} r^{-(d+2s)} \mathbf{1}_{[0.1,6]}(x) + 10^{d+2s} \mathbf{1}_{[0,0.1]}(x)$ 
with exponent $s = 0.5$ and $d = 1$, where $c_{d,s} = 4^s\pi^{-d/2}\Gamma(d/2+s) \Gamma(-s)$. 
It is almost singular with multiscale values and its values near the singularity are crucial to the operator. It has $R =6.51$. 
\vspace{-1mm}
\begin{figure}[h!]\vspace{-3mm}
{\includegraphics[width =1.0\textwidth]{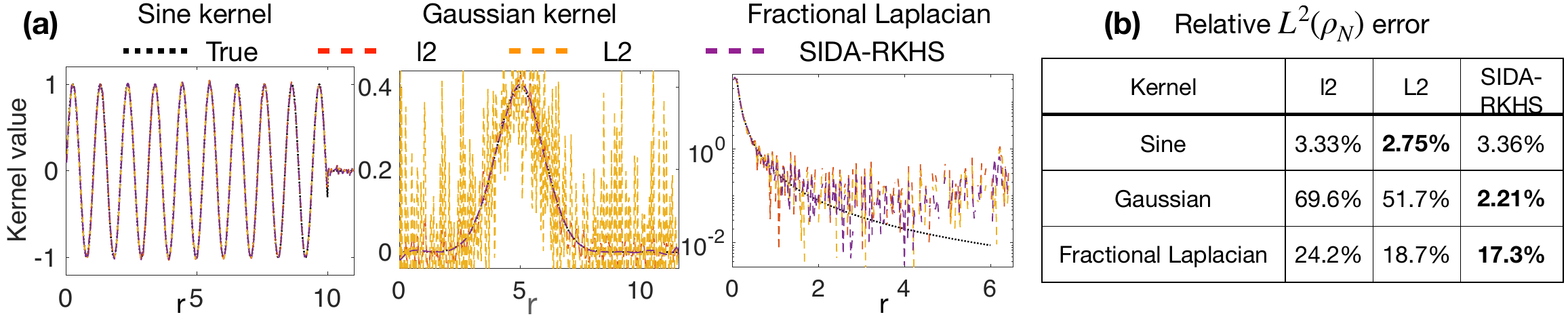}}\vspace{-6mm}
\caption{(a): Typical estimators from noisy data with noise-to-signal-ratio $\mathrm{nsr} =1$ and $\Delta x =0.025$. 
(b): the relative $L^2(\rho_N^J)$ errors of these estimators. Bold numbers highlight the best method. The SIDA-RKHS regularizer consistently obtains accurate estimators in all three cases. 
}\label{fig:typicalEstftn} 
\end{figure}\vspace{-1mm}

\textbf{Performance of the regularizers.} 
We present the typical estimators and the convergent rate of the estimator as data mesh refines.  Figure \ref{fig:typicalEstftn} shows typical estimators for the three examples from noisy data with a noise-to-signal ratio  nsr=1 and $\Delta x = 0.025$. The hypothesis space's 
dimension is selected by minimal loss value. All three regularizers are able to estimate  the Sine kernel accurately and the Fractional Laplacian kernel reasonably. The SIDA-RKHS regularizer significantly outperforms the regularizers with $l^2$ or $L^2$-norm in the example of the Gaussian kernel. 

	\vspace{-1mm} 
\begin{figure}[H]	\vspace{-2mm} 
    \centering 	
    {\includegraphics[width =1.0\textwidth]{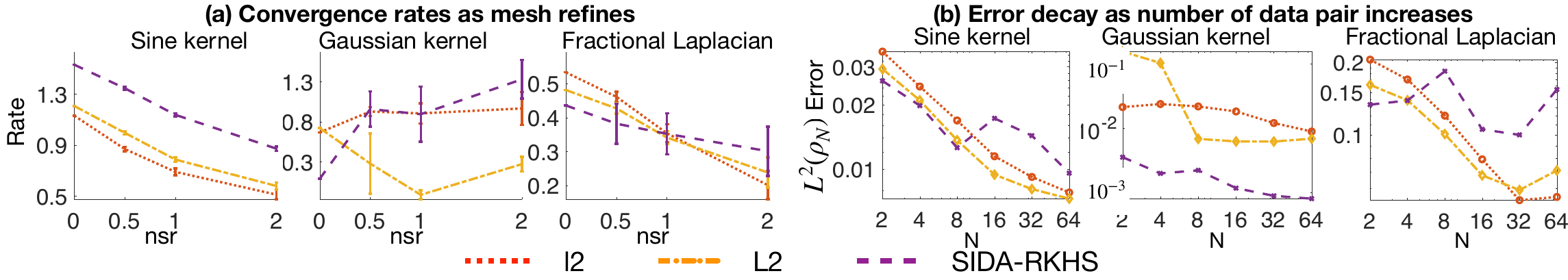}}
	\vspace{-6mm} 
\caption{(a) The means and standard deviations of the convergence rates as mesh refines in 100 independent simulations. The SIDA-RKHS regularizer obtains consistent rates for noisy data. (b) Error decay as the number of data pairs increases when $nsr=1$ and $\Delta x=0.0125$. } \label{fig:convMeanStd2}		\vspace{-5mm}
\end{figure}
The SIDA-RKHS regularizer's superior performance is further validated by the rates of convergence when $\Delta x$ decreases, from 100 independent simulations with noises with $nsr\in \{0, 0.5,1,2\}$, as shown in Figure \ref{fig:convMeanStd2}. The SIDA-RKHS regularizer has rates generally higher than those of the other two regularizers when the data gets more noisy. We note that the SIDA-RKHS regularizer has deceivingly lower rates for noiseless data, even though it actually has more accurate estimators (see Figure \ref{fig:convftn} in appendix). Thus, the goal is to seek an accurate estimator with a consistent rate. Here the rates for the smooth kernels are higher than the rate for the singular kernel, because the order of numerical error in the Riemann sum integrator are higher (see \cite{fan2021asymptotically}). 

\textbf{Increasing the number of data pairs.} Sine the operator is linear, only linearly independent data brings new information for the learning. Figure \ref{fig:convMeanStd2}(b) shows that as $N$ of data $\{u_i(x)\}_{i=1}^N = \{\sin(ix),\cos(ix)\}_{i=1}^{N/2}$ increases, the estimators become more accurate but without a convergence rate. Note that the data pairs do not provide independent (random) samples of $\rho_N$, which also varies with data. Thus, our learning problem is fundamentally different from regression for random samples and we do not expect a convergence rate $N^{-1/2}$.  An interesting future direction is to design experiments to collect informative data to enlarge the FSOI and accelerate the convergence.


In summary, the SIDA-RKHS regularizer consistently obtains accurate convergent estimators when data mesh refines for either noiseless or noisy data. On the contrary, the regularizers with $l^2$ norm or $L^2$ norm, are not robust to noise and may fail to converge, due to their negligence of the FSOI.

\vspace{-1mm}
\subsection{Homogenization of wave propagation in meta-material} \label{sec:metamaterial}

\begin{figure}[h!]
\centering
\subfigure{\includegraphics[width=1.\columnwidth]{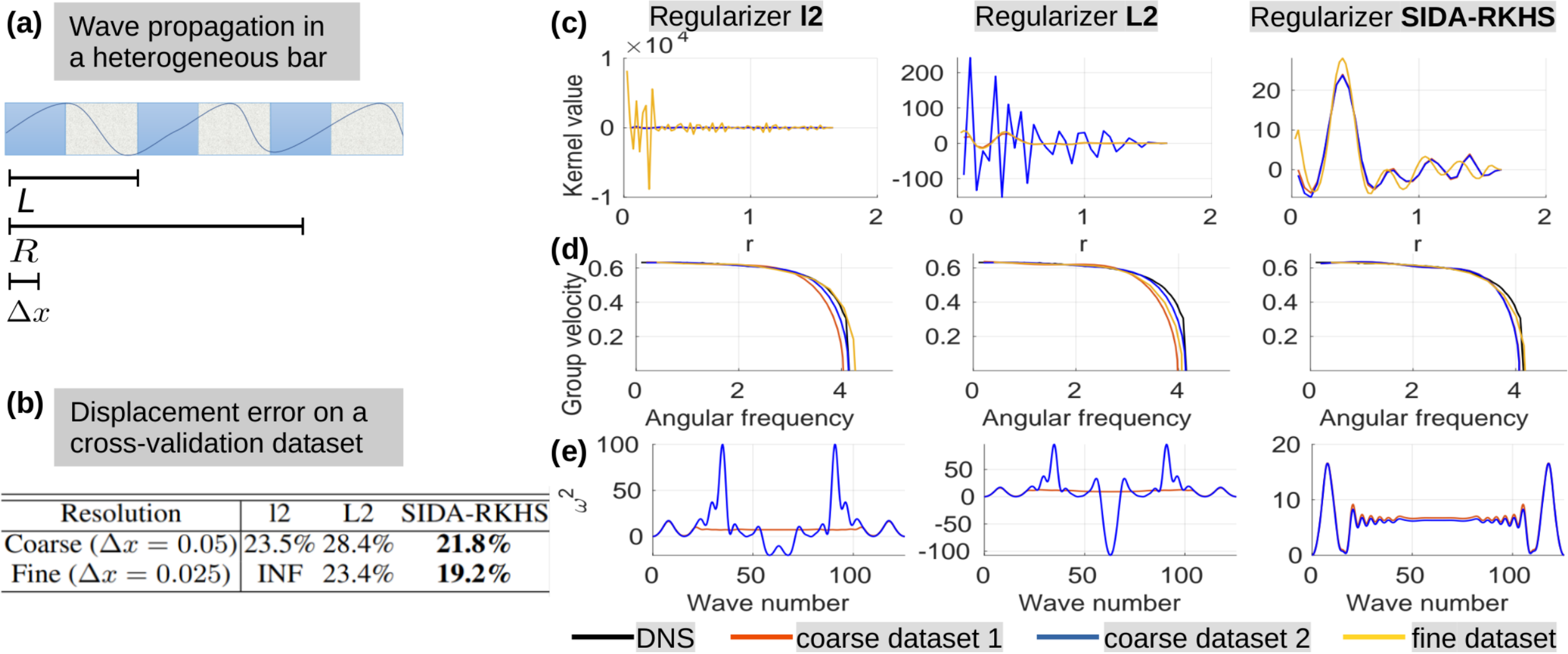}}    \vspace{-6mm}
\caption{\small Real-world application: wave propagation in a heterogeneous bar with ordered microstructure of period $L=0.4$, and the estimated support of the kernel has a bound $R=1.65$. 
\vspace{-3mm}
}\label{fig:bicrack_pa}
\end{figure}

We seek a nonlocal homogenized model for the stress wave propagation in a one-dimensional heterogeneous bar with a periodic microstructure. 
For this problem, the goal is to obtain an effective surrogate model from high-fidelity (HF) datasets generated by solving classical wave equation, acting at a much larger scale than the size of the microstructure. Differing from previous examples, this problem has no ground-truth kernel. Therefore, we evaluate the estimator by measuring its effectiveness of reproducing HF data in applications that are subject to different loading conditions with a much longer time from the problems used as training data. 

For both training and validation purposes we use the HF dataset generated by the direct numerical solver (DNS) introduced in \cite{silling2020_PropagationStressPulse}, which provides exact solutions of velocities including the appropriate jump conditions for the discontinuities in stress that occur at waves. Although the DNS has high accuracy on wave velocity, it is not suitable for long-term prediction because it requires the modeling of wave propagation through thousands of microstructural interfaces, which makes the computational cost prohibitive.
To accelerate the computation,  we approximate the HF model by a nonlocal model:
\begin{equation}\label{eq: NLmodel}
\partial_{tt} {u}(x,t)-L_\phi[{u}](x,t)=g(x,t), \text{ for }(x, t) \in \Omega \times [0, T],
\end{equation}
where $L_\phi$ is a nonlocal operator in the form of \eqref{eq:nonlocal_op} with a kernel $\phi$ being supported in $[0, R]$.

\textbf{Experiment settings.}
We consider four types of data: three for training and one for validation of our algorithm. Three types of training datasets are employed: In Type 1 dataset, the bar is subject to an oscillating source $g(x,t)$; In Type 2 dataset, a boundary velocity loading \e{\partial_t {u}(-50,t) = \cos(j t)} is applied;  In Type 3 dataset, all settings are the same as in Type 2, except that the $\cos(jt)$ type loading is replaced by $\sin(jt)$. In all training datasets we consider a relatively small domain $\Omega=[-50,50]$ and short time $t\in[0,2]$. Two spatial resolutions, $\Delta x=0.05$ and $\Delta x=0.025$ are considered, which we denote as the ``coarse'' and ``fine'' datasets, respectively.

With these three types of training datasets, we design three experiment settings to validate our method: 
\newline
$\bullet$ \emph{Coarse dataset} 1: 
    we train the estimator using ``coarse'' dataset of Types 1 and 2.
\newline
$\bullet$ \emph{Coarse dataset} 2: 
    we train the estimator using ``coarse'' dataset of Types 1 and 3. By comparing the learnt estimator from this setting with the result from setting 1, we mean to investigate the sensitivity of the inverse problem with respect to the choice of datasets.
\newline
$\bullet$ \emph{Fine dataset:} 
    we train the estimator using ``fine'' dataset of Types 1 and 2. By comparing the learnt estimator from this setting with the result from setting 1, we aim to check the convergence of the estimator with increasing data resolution. Note that the problem might becomes more ill-posed when decreasing $\Delta x$. Therefore, proper regularization is expected to become more important. 

Additionally, we create a validation dataset, denoted as Type 4 dataset, very different from the training dataset. It considers a much longer bar ($\Omega=[-133.3,133.3]$), under a different loading condition from the training dataset, and with a $50$ times longer simulation time ($t\in [0,100]$). Therefore, the cross-validation error checks the generalizability of the estimators.

\textbf{Results assessment.} We present the learnt estimators in Figure \ref{fig:bicrack_pa}. Since there is no ground-truth kernel, we assess the performance of each estimator based on three criteria. Firstly, we report in Figure \ref{fig:bicrack_pa}(b) the prediction $L^2$ error of displacement on the cross validation dataset at $T=100$. 
Secondly, we report in Figure \ref{fig:bicrack_pa}(d) the resultant estimators 
the group velocity curves from our model and compare them with the curves computed with DNS. These curves directly depicts how much our surrogate model reproduces the dispersion properties in the heterogeneous material. At last, the learnt model should provide a physically stable material model. To check this, we also report the dispersion curve in  \ref{fig:bicrack_pa}(e). Its positivity indicates that the learnt nonlocal model is physically stable.

\textbf{Performance of the estimators.} Comparing the three estimators in Figure \ref{fig:bicrack_pa}(c), one can see that only the SIDA-RKHS regularizer obtains consistent estimators in all three experiment settings. 
The oscillatory estimators of regularizers with $l^2$ or $L^2$-norm verify the ill-posedness, and highlight the importance of using proper regularizers in nonlocal operator learning methods. The dispersion curves in Figure \ref{fig:bicrack_pa}(e) stress the importance of regularizer from another aspect of view: our SIDA-RKHS regularizer provides physically stable material models in all settings, while the regularizers with $l^2$ or $L^2$-norm may result in highly oscillatory and non-physical models.

We further examine the regularized estimator in terms of its capability in reproducing DNS simulations through the prediction error of $u$ on the cross validation dataset. 
When $\Delta x=0.025$, it takes about $48$ hours for the DNS simulation to generate one sample, while the homogenized nonlocal model only requires less than $20$ minutes. From \ref{fig:bicrack_pa}(b), we can see that when $\Delta x=0.05$, all three regularizers are robust and able to reproduce the DNS simulation with a reasonable accuracy ($\sim 20\%$). When we increase the data resolution to $\Delta x=0.025$, the estimated nonlocal model from $l^2$ regularizer becomes unstable, which again verifies our analysis: when the data mesh refines, the kernel learning problem becomes more ill-posed and a good regularizer becomes a necessity. Meanwhile, both the $L^2$ and SIDA-RKHS regularizers lead to a more accurate estimator, indicating a trend of convergence. 
On both datasets, the SIDA-RKHS regularizer obtains the most accurate estimators.

\vspace{-3mm}
\subsection{Limitations and future directions} \label{sec:limitations}
\vspace{-1mm}
\emph{Non-radial high-dimensional kernels.} 
When the kernel is radial, our algorithm is readily applicable to higher dimensions (see Appendix \ref{sec:append_alg}). When the kernel is non-radial high-dimensional, however, the regression will face the well-known curse-of-dimensionality, but our identifiability theory remains valid. Thus, a future direction is to utilize methods such as kernel-regression or neural networks and further develop the SIDA-RKHS regularization. 
\\
\emph{ Convergence analysis.} We have obtained convergent regularized estimators, but a convergence analysis is left as future work. The main difficulty to overcome is the complex combination of three factors: operator spectrum decay, the errors from numerical integration and noise, and regularization. 
 \vspace{-4mm}
\section{Conclusion}\label{sec:conclusion} \vspace{-1mm}
We have characterized the identifiability pitfall in the learning of kernels in nonlocal operators, and proposed a new regularization method to fix this issue and achieve estimator convergence. In particular, we have established a rigorous identifiability theory for the nonparametric learning of kernels in nonlocal diffusion operators, specifying the function space of identifiability. Based on the theory, we have introduced a nonparametric regression algorithm equipped with a data adaptive RKHS regularization method. Tests on synthetic and real-world datasets show that the our algorithm consistently obtains accurate and convergent estimator, outperforming common benchmark regularizers. 
Our method addresses the critical estimator diverging phenomena observed in previous nonlocal operator learning methods, and the proposed framework provides a promising new direction towards overcoming the ill-posedness to achieve convergence in operator learning.

\newpage

\appendix

\section{Proofs}\label{sec:proofs}

\begin{proof}[Proof of Lemma {\rm \ref{lemma:Gbar}}] 
Part (a) follows directly from the definition of $\Gbar$. Recall that a bivariate function $\Gbar$ is positive semi-definite iff for any $(c_1,\ldots,c_m)\in \R^m$ and any $\{r_j\}_{j=1}^m\subset \R^{d}$, the sum $\sum_{k=1}^m\sum_{j=1}^m c_kc_j\Gbar(r_k,r_j)\geq 0$ (see e.g. \cite{BCR84,Cucker2007,LLMTZ21}). Then, noting that from \eqref{eq:G} and \eqref{eq:Gbar}  we have 
\begin{align*}
& \sum_{k=1}^m\sum_{j=1}^m c_kc_j\Gbar(r_k,r_j) \\
= &  \frac{1}{N}\sum_{i=1}^N\int_{|\eta | =1} \int_{|\xi |=1}  \left[ \int \sum_{k=1}^m\sum_{j=1}^m c_kc_j\frac{[u_i(x+r_k\xi) - u_i(x) ] [u_i(x+r_j\eta) - u_i(x) ]}{\rho_N'(r_j)\rho_N'(r_k)}  dx \right]d\xi d\eta \\
= & \frac{1}{N}\sum_{i=1}^N\int_{|\eta | =1} \int_{|\xi |=1}  \left[ \int \left|\sum_{k=1}^m c_k\frac{[u_i(x+r_k\xi) - u_i(x) ] }{\rho_N'(r_k)}  dx\right|^2 \right]d\xi d\eta \geq 0.
\end{align*}
Thus, $\Gbar$ is positive semi-definite. 

For Part (b), the operator $\LGbar$ is compact because  $\Gbar\in L^2(\rho_N\times \rho_N)$, which follows from the fact that each $u_i$ is bounded (thus, $\Gbar$ is also bounded). Also, since $\Gbar$ is positive semi-definite, so is $\LGbar$. The equation \eqref{eq:dbinnerp_LG} follows from \eqref{eq:binlinearForm_int}. 

Part (c) is a standard operator characterization of the RKHS $H_G$ (see e.g., \cite{Cucker2007}). 

For Part (d), the eigenfunctions are orthonormal and the eigenvalues decay to zero because the operator $\LGbar$ is positive semi-definite and compact, as shown in Part (b). The first equation in \eqref{eq:norms} follows from \eqref{eq:dbinnerp_LG}, and the second equation follows from the orthogonality of the eigenfunctions. At last, if $\phi\in H_G$, by the characterization of $H_G$'s inner product in Part (c), we have the third equation in \eqref{eq:norms}. 
\end{proof}

\begin{proof}[Proof of Lemma {\rm \ref{lemma:D_lossFn}}]
 Recall that with the bilinear form $\dbinnerp{\cdot,\cdot}$, defined in \eqref{eq:binlinearForm}, 
 we can rewrite the loss functional as 
 \begin{equation}\label{eq:lossFn2}
 \begin{aligned}
 \calE(\phi)
 & = \dbinnerp{\phi,\phi} -  \frac{1}{N}\sum_{i=1}^N \int 2 L_\phi[u_i](x) f_i(x)dx + C_f,
\end{aligned}
\end{equation}
 where $C_N^f = \frac{1}{N}\sum_{k=1}^N  \int |f_i (x) |^2 dx $. Then, the derivative $\nabla \calE(\phi)$ follows from \eqref{eq:dbinnerp_LG} and a rewriting of the bilinear form: 
\begin{align}
& \dbinnerp{\phi_1,\phi_2}=  \frac{1}{N}\sum_{i=1}^N \int \left[ \int \int \phi_1(|z|) [u_i(x+z) - u_i(x) ]\phi_2(|y|) [u_i(x+y) - u_i(x) dydz]  \right] dx  \notag \\
 = & \frac{1}{N}\sum_{i=1}^N \int \int \phi_1(|z|)\phi_2(|y|)  \left[ \int  [u_i(x+z) - u_i(x) ] [u_i(x+y) - u_i(x) ] dx \right] dydz \notag  \\
 = &  \int_0^\infty \int_0^\infty \phi_1(r)\phi_2(s) G(r,s) drds =   \int_0^\infty \int_0^\infty \phi_1(r)\phi_2(s) \Gbar (r,s) \rho_N(dr)\rho_N(ds),  \label{eq:binlinearForm_int}
\end{align}
with $G$ and $\Gbar$ given in \eqref{eq:G} and \eqref{eq:Gbar}, where the last equality is a re-weighting by $\rho_N$. 
\end{proof}

\begin{proof}[Proof of Theorem {\rm \ref{thm:space_ID} }] 
By Lemma \ref{lemma:D_lossFn}, the Fr\'echet derivative of the loss functional is $\nabla \calE(\phi) =  2 ( \LGbar \phi - \phi_N^f)$. Thus, the loss functional has a unique minimizer only in the function space where $\nabla \calE(\phi)$ has a unique zero, that is, the operator $\LGbar$ has an inversion. The largest such a space is the eigenspace expanded by all eigenfunctions with non-zero eigenvalues of $\LGbar$. Furthermore, projecting $\phi_N^f$ to this sapce, we have the the minimizer $\widehat \phi = \LGbar^{-1} P\phi_N^f $ as given in the theorem. 
\end{proof}

\section{Algorithm: nonparametric regression with SIDA-RKHS regularization}\label{sec:append_alg}

In this section we provide detailed description of the algorithm proposed in Section \ref{sec:alg}. 

Our algorithm consists of three steps. 
First, we utilize the data to estimate the exploration measure and the support of the kernel. Based on them, we set a class of hypothesis spaces, with their dimensions  i.e., the number of basis functions, in a proper range moving from under-fitting to over-fitting.  
Second, we assemble the regression matrices vectors from data for each of the hypothesis spaces. 
Finally, we identify the estimators with SIDA-RKHS regularization for these hypothesis spaces and select the one with the best fitting.

To start, we assume that the discrete data $\{u_i(x_j), f_i(x_j)\}_{i = 1}^{N}$ comes with equidistant mesh points $\{x_j = j\Delta x\}_{j=0}^J$. For simplicity, we consider only the 1D case, and the extension to multi-dimensional cases is straightforward. We note that the current problem setting assumes data on mesh points, thus the data size increases exponentially as the dimension increases, which is the well-known curse-of-dimensionality. To overcome this curse-of-dimensionality, one can consider other settings with mesh-free representation of data by random samples and a loss functional based on expectations (see, e.g. \cite{LangLu22}), and this is beyond the scope of the current study.  

\paragraph{Step 1: Set a class of hypothesis spaces.} We set a class of data-adaptive hypothesis spaces with their dimensions set to range from under-fitting to over-fitting. The key is the exploration measure and the support of the kernel estimated data. The exploration measure $\rho_N^J$ is computed from data as in \eqref{eq:rho_disc}, which uses only the information from $u_i$.  To estimate the support of the kernel, we extract the additional information from $\{f_i\}$ as follows. We set the data-adaptive support of the kernel to be $[0,R]$ with $R$ defined by 
 \begin{equation}\label{eq:suppKernel}
R= 1.1 \min\{R_\rho, \max \{  |L^f_i - L^u_i|,|R^f_i - R^u_i|\}_{i=1}^N \}, 
 \end{equation}
where $(L^u_i, R^u_i )$ and $(L^f_i, R^f_i )$ are the lower and upper bounds of the supports $\supp(u_i)$ and $\supp(f_i)$ respectively, and $R_\rho$ is the maximum of the support of $\rho_N^J$. 
That is, the support of the kernel lies inside the support of the exploration measure, and it is the maximal interaction range indicated by the difference between supports of $u_i$ and $f_i$, which extracts the additional information in the data $\{f_i\}$. Here the multiplicative factor 1.1 is an artificial factor to enlarge the range, so that the supports of the basis functions will fully cover the explored region. To avoid unbounded support in the data-based estimation in \eqref{eq:suppKernel}, in numerical experiments we set a threshold to be $10^{-8}$ when estimating supports of $u_i$, $f_i$ and $\rho_N^J$. This truncation narrows the interaction range.

The estimated support of the kernel is the region explored by data. Outside of the region, the data provides little information about the kernel. Thus, we focus on learning the kernel in this region and set the local basis functions to be supported in it. Furthermore, we constrain the exploration measure to be supported in $[0,R]$. For simplicity of notation, we still denote it by $\rho_N^J$ or $\rho_N$. 

With the exploration measure and the support of the kernel, we select a class of  basis functions $\{\phi_i\}_{i=1}^n$ and a range of $n$ for the hypothesis space $\mH_n =\mathrm{span}\{\phi_i\}_{i=1}^n$. The basis function can be either global basis functions such as Bernstein polynomials as those used in  \cite{you2020_DatadrivenLearning,you2021_DatadrivenLearning} and trigonometric functions, or local basis functions such B-spline polynomials (see Appendix \ref{sec:Bsplines} for a brief introduction). We focus on local basis functions because they are more flexible to adaptive to local structure of the kernel. To set the range for $n$, we note that the mesh points of the kernel's independent variable explored by data are $\{k\Delta x:  k=1,\ldots, \floor{\frac{R}{\Delta x}}\}$. Meanwhile, the basis function should be linearly independent in $L^2(\rho_N^J)$ so that the basis matrix  
\begin{equation}\label{eq:Bmat}
B_n = ( \innerp{\phi_i,\phi_j}_{L^2(\rho_N^J)})_{1\leq i,j\leq n} \in \R^{n\times n}
\end{equation}
is non-singular. Thus, we set the range of $n$ to be in $\floor{\frac{R}{\Delta x}}\times [0.2,1]$ such that $B_n$ is non-singular while covering a wide range of dimensions. For example, when we use piecewise constant basis, we can set $n=\floor{\frac{R}{\Delta x}}$, and we get $B_n=\mathrm{Diag}(\rho_N^J)$. Thus, we estimate the kernel as a vector of its values on the mesh points, with $L^2(\rho_N^J)$ being a vector space with a discrete-measure $\rho_N^J$.

\paragraph{Step 2: Assemble regression matrices and vectors.} We assemble the regression matrix $\Abar_n$ and vector $\bbar_n$, as defined in \eqref{eq:Ab}, for each hypothesis spaces $\mH_n =\mathrm{span}\{\phi_i\}_{i=1}^n$. Together with the basis matrix $B_n$ in \eqref{eq:Bmat}, the triplet $(\Abar_n, \bbar_n, B_n)$ is all we need for regression with SIDA-RKHS regularization in the next step. 
 
To avoid repeated reading of data, we extract the regression data that can be used for all hypothesis spaces by utilizing the regression structure, which requires reading the data only once. Note that to compute $\Abar_n(i,j) = \dbinnerp{\phi_i,\phi_j}$ for any pair of basis functions, with the bilinear form defined in \eqref{eq:binlinearForm_int}, we only need $G$ defined in \eqref{eq:G}. 
We note that when $d=1$, the integral $\int_{|\eta | =1} g(\eta)d\eta = g(\eta) + g(-\eta)$, therefore, we have 
\begin{align}\label{eq:G1D}
G(r,s) = \frac{1}{N}\sum_{i=1}^N \int  [u_i(x+r) +u_i(x-r) -2 u_i(x) ] [u_i(x+s)+u_i(x-s) - 2u_i(x) ] dx 
\end{align}
for $r,s\in \mathrm{supp}(\rho_N)$. Similarly, for a basis function $\phi_i$, to compute $\bbar(i)$ in \eqref{eq:Ab}, which can be re-written as  
$\bbar_n(i)=  \frac{1}{N}\sum_{k=1}^N \int  L_{\phi_i}[u_k](x) f_k(x)dx  = \int_0^R \phi_i(r) g_N^f(r) dr
$, 
we only need the function $g_N^f$ defined by 
\begin{equation}\label{eq:g_N^f}
g_N^f(r) =  \frac{1}{N}\sum_{i=1}^N \int_\Omega\int_{|\xi|=1} [u_i(x+r\xi) - u_i(x) ] f_i(x)d\xi\,dx. 
\end{equation}
Let $r_k= k\Delta x$ for $k = 1, \ldots,  \floor{\frac{R}{\Delta x}}$, which are all the mesh points the data explore. Then, all the regression data  we need  in the original data \eqref{eq:data} are 
\begin{equation}\label{eq:regressionData}
 \left\{G(r_k,r_l), g_N^f(r_k), \rho_N^J(r_k), \text{ with } k,l = 1, \ldots,  \floor{\frac{R}{\Delta x}}\right\},
\end{equation}
where $G$, $g_N^f$ and $\rho_N^J$ are defined respectively in \eqref{eq:G}, \eqref{eq:g_N^f} and \eqref{eq:rho_disc}.  

With these regression data, the triplet $(\Abar_n, \bbar_n, B_n)$ can be efficiently evaluated for any basis functions using a numerical integrator to approximate the corresponding integrals. For example, with Riemann sum approximation, we compute the normal matrix $\Abar_n$ and vector $\bbar_n$ and the basis matrix $B_n$ as   
\begin{equation}\label{eq:AbB}
\begin{aligned}
\Abar_n(i,j) & = \dbinnerp{\phi_i,\phi_j} \approx \sum_{k,l} \phi_i(r_k)\phi_j(r_l) G(r_k,r_l)) \Delta x^2, \\
\bbar_n(i)  & \approx \sum_{k} \phi_i(r_k)g_N^f(r_k)) \Delta x, \\
B_n(i,j) & \approx \sum_{k} 
\phi_i(r_k)\phi_j(r_k) \rho_N^J(r_k) \Delta x.  
\end{aligned}
\end{equation}
\paragraph{Step 3: Regress with SIDA-RKHS regularization.} Our SIDA-RKHS regularization method uses the norm of the SIDA-RKHS so as to ensure the learning to take space in the function space of identifiability as discussed in Section \ref{sec:Identifiability}. That is, our estimator is the minimizer of the regularized loss in \eqref{eq:lossFn_reg} with the regularization norm $\mathcal{R}(\phi) = \|\phi\|^2_{H_G}$ defined in \eqref{eq:norms}.   

\noindent\emph{Computation of the RKHS norm}. We can effectively approximate the RKHS norm $\|\phi\|^2_{H_G}$ using the triplet $(\Abar_n, \bbar_n, B_n)$. It proceeds in two steps. First, we solve the generalized eigenvalue problem $\Abar_n V =  B_n  V\Lambda$, where $\Lambda$ is a diagonal matrix of the generalized eigenvalues and the matrix $V$ has columns being eigenvectors orthonormal in the sense that $V^\top B_n V = I_n$. Here these eigenvalues approximate the eigenvalue of $\LGbar$ in \eqref{eq:LG}, and $\widehat\psi_k = V_{jk}\phi_j$ approximates the eigenfunctions of $\LGbar$. Then, we compute the square RKHS norm of $\phi= \sum_i c_i \phi_i$ as
\begin{equation}\label{eq:rkhsNormMat}
\|\phi\|^2_{H_G} = c^\top B_{rkhs} c, \, \text{ with } B_{rkhs} = (V\Lambda V^\top)^{-1}, 
\end{equation}
 where the inverse is taken as pseudo-inverse, particularly when $\Lambda$ has zero eigenvalues. 
 
With the RKHS-norm ready, we write the regularized loss for each function $\phi= \sum_i c_i \phi_i$ as 
$
\calE_\lambda(\phi) = c^\top (\Abar_n + \lambda B_{rkhs}) c - 2 c^\top \bbar_n + C_N^f. 
$
The regularized estimator is 
\begin{align}\label{eq:c_reg}
\widehat{\phi_\lambda} = \sum_{i = 1} ^ n c^i_\lambda \phi_i, \ c_\lambda = (\Abar_n + \lambda B_{rkhs})^{-1}\bbar_n. 
\end{align}

We will select the hyper-parameter that balances the loss $\calE$ and the regularization term by the widely-used L-curve method \cite{hansen_LcurveIts_a}. It identifies the optimal hyper-parameter as the maximizer of the curvature of the curve (see Section \ref{sec:append_Lcurve}).  

\section{B-spline basis functions and the L-curve method}
\vspace{-1mm}
\subsection{B-spline basis functions}\label{sec:Bsplines} 
B-spline is a class of piecewise polynomials, and is capable of representing the local information of the target function. Here we review briefly the recurrence definition and properties of the balanced B-splines, for more details we refer to the Chapter 2 of \cite{piegl1997_NURBSBook} and \cite{lyche2018_FoundationsSpline}.    

Given a non-decreasing sequence of real numbers $\{r_0,r_1,\ldots, r_m\}$ (called knots), the B-spline basis functions of degree $p$, denoted by $\{N_{i,p} \}_{i=0}^{m-p}$, is defined recursively as
\begin{equation}\label{eq:B_spline}
\begin{aligned}
  &N_{i,0}(r) = 
  \left\{
    \begin{array}{lr}
      1,\ &r_i \leq r < r_{i+1},\\
      0,\ &otherwise,
    \end{array}
  \right.\\
  &N_{i,p}(r) = \frac{r - r_i}{r_{i + p} - r_i} N_{i, p-1}(r) + \frac{r_{i+p+1} - r}{r_{i + p + 1} - r_{i + 1}}N_{i + 1, p- 1}(r).
\end{aligned}
\end{equation}
The B-spline basis has the following properties: 
\begin{itemize}[leftmargin=*]\setlength\itemsep{-1mm}
\item Each function $N_{i,p}$ is a nonnegative local polynomial of degree $p$, supported on $[r_i,r_{i+p+1}]$;
\item At a knot with multiplicity $k$, it is $p-k$ times continuously differentiable. Hence, the smoothness increases with the degree but decreases when the knot multiplicity increases; 
\item The basis satisfies partition unity: for each $r\in [r_i,r_{i+1}]$, $\sum_{j} N_{j,p}(r)  = \sum_{j=i-p}^i N_{j,p}(r) =1$. 
\end{itemize}

 We set the knots to be a uniform partition of the support of $\rhoT$, $[R_{min}, R_{max}]$,  
\[  R_{min} = r_{0} \leq r_1 \leq \cdots \leq r_{m}= R_{min}.\]
We set the basis functions of the hypothesis $\mH $, whose dimension is $n=m-p$, to be 
$$\phi_i(r) = N_{i, p}(r), \ i = 1,\dots, m-p. $$ 
Thus, the basis functions $\{\phi_i\}$ are piecewise degree-$p$ polynomials with knots adaptive to $\rhoT$. 

\vspace{-1mm}
\subsection{Hyper-parameter selection by the L-curve method}\label{sec:append_Lcurve}
We select the parameter $\lambda$ by the L-curve method \cite{hansen_LcurveIts_a,LangLu22}. 
Let $l$ be a parametrized curve in $\R^2$: 
$$ l(\lambda) = (x(\lambda), y(\lambda)) := (\text{log}(\calE(\widehat{\phi_\lambda}), \text{log}(\mathcal{R} (\widehat{\phi_\lambda} )), $$
where $\calE(\widehat{\phi_\lambda}) = c_\lambda^\top \Abar_n c_\lambda - 2c_\lambda^\top \bbar_n-C_N^f$, and $ \mathcal{R} (\phi)$ is the regularization term, for example, $\mathcal{R} (\widehat{\phi_\lambda} ) =  \| \widehat{\phi_\lambda}\|_{H_\Gbar}^2 = c_\lambda^\top B_{rkhs} c_\lambda$.
The optimal parameter is the maximizer of the curvature of $l$. In practice, we restrict $\lambda$ in the spectral range $[\lambda_{min},\lambda_{max}]$ of the operator $\LGbar$, 
\begin{align}\label{eq:opt_lambda}
	\lambda_{0} 
	= \argmax{\lambda_{\text{min}} \leq \lambda \leq \lambda_{\text{max}}}\kappa(l(\lambda)) 
	= \argmax{\lambda_{\text{min}} \leq \lambda \leq \lambda_{\text{max}}}
	\frac{x'y'' - x' y''}{(x'\,^2 + y'\,^2)^{3/2}},
\end{align}
where $\lambda_{min}$ and $\lambda_{max}$ are computed from the smallest and the largest generalized eigenvalues of $(\Abar_n,B_n)$. 
This optimal parameter $\lambda_{0}$ balances the loss $\calE$ and the regularization (see \cite{hansen_LcurveIts_a} for more details).
In practice, instead of computing the second order derivatives, we compute the curvature by the reciprocal of the radius of the interior circle of three consecutive points\footnote{Are Mjaavatten (2022). Curvature of a 1D curve in a 2D or 3D space (\url{https://www.mathworks.com/matlabcentral/fileexchange/69452-curvature-of-a-1d-curve-in-a-2d-or-3d-space}), MATLAB Central File Exchange.}.

\section{Additional numerical results for synthetic data examples}\label{sec:additionResults}
This section provides additional numerical results for the examples with synthetic data. 

Figure \ref{fig:convftn} shows that the SIDA-RKHS regularizer leads to converging estimators in all three examples for both noisy and noiseless data, whereas the $l^2$-norm and the $L^2$-norm regularizers' estimators have slow convergent rates or even no convergence when the data is noisy. 

\begin{figure}[h!]
    \centering 	  \hspace{-4.5mm}
\includegraphics[width =0.65\textwidth]{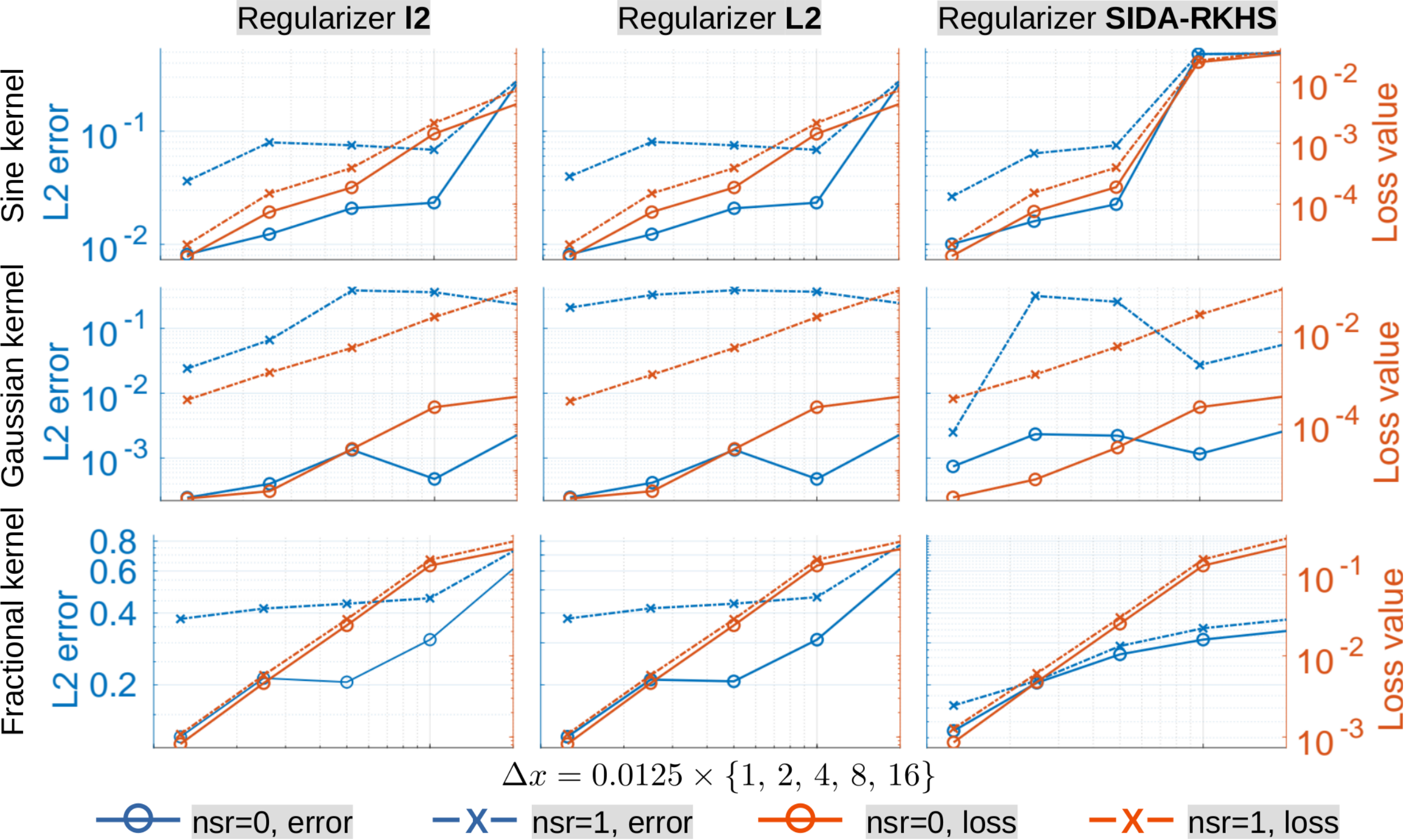} 
\vspace{-2mm}
\caption{Convergence of function estimators as the data mesh-size $\Delta x$ refines, along with values of the loss function. The SIDA-RKHS regularizer consistently converges for both noiseless and noisy data, with better rates (slope) than the other two regularizers for noisy data. Note that for the fractional kernel, it has a lower rate though being more accurate. 
} \label{fig:convftn}		\vspace{-3mm}
\end{figure}

We note that the performance of these regularizers depends on the optimal regularization strength $\lambda_0$, which is selected by the L-curve method introduced in Section \ref{sec:append_Lcurve}. In our tests, all regularizers can successfully select the optimal $\lambda_0$ for most of the time, and the SIDA-RKHS regularizer has the most well-shaped L-curve, which leads to the most robust regularization (see Figure \ref{fig:Lcurve} for typical L-curve plots).    
\begin{figure}[h!]
    \centering 	  \hspace{-4.5mm}
\includegraphics[width =1.0\textwidth]{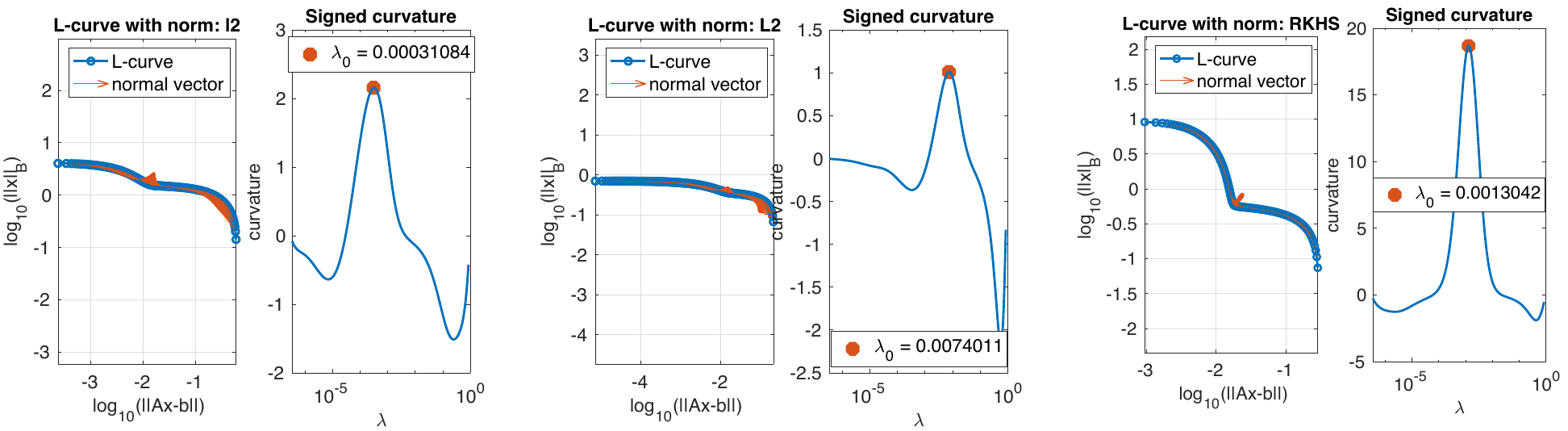} 
\vspace{-2mm}
\caption{Typical L-curve plots for the selection of the optimal regularization parameter $\lambda_0$ for the Gaussian kernel with $\Delta x=0.05$ and $nsr=1$. From left to right: the $l^2$, $L^2$ and SIDA-RKHS regularizers. All regularizers successfully select the optimal $\lambda_0$, and the SIDA-RKHS regularizer has the most well-shaped L-curve.   
} \label{fig:Lcurve}		\vspace{-3mm}
\end{figure}

\section{Detailed Real-world Dataset Experiment Settings}\label{sec:barsetting}

In this section we provide further experiment details for the real-world dataset studied in  \ref{sec:metamaterial}.

For both training and validation purposes we generate data using high-fidelity (HF) simulations for the propagation of stress waves within the microstructure of the heterogeneous, linear elastic bar. In the following, we use $\hat{u}$ to denote the HF solution, to distinguish the HF dataset from the homogenized solution of \eqref{eq: NLmodel}. The HF-model is a classical wave equation: the displacement $\hat{u}(x, t)$ satisfies, for $(x, t) \in \Omega \times [0, T]$ with $\Omega \subset \R$, 
\begin{align}\label{eq: HF model}
   \partial_{tt} \hat{u}(x,t)-L_{HF}[\hat{u}](x,t)=g(x,t), 
\end{align}
with a force loading term $g(x,t)$, proper boundary conditions and initial conditions $\hat{u}(x,0)=0$, $\partial_t \hat{u}(x,0)=0$. 
Considering the heterogeneous bar of two materials depicted in Figure \ref{fig:bicrack_pa}, \eqref{eq: HF model} describes the stress wave propagating with speed $c_1=\sqrt{E_1/\rho}$ in material 1 and speed $c_2=\sqrt{E_2/\rho}$ in  material 2. We solve the HF-model \eqref{eq: HF model} by the direct numerical solver (DNS) introduced in \cite{silling2020_PropagationStressPulse}. 
The DNS employs the characteristic line method, which provides exact solutions of velocities. For each grid point $x_j\in\Omega$ at time step $t^n=n\Delta t$, where $\Delta t$ is the time step size, with the calculated exact velocity $\hat{v}(x_j,t^n)$ and the estimated displacement from the last time step $\hat{u}(x_j,t^{n-1})$ we update the HF displacement by
$$\hat{u}(x_j,t^{n})=\hat{u}(x_j,t^{n-1})+\Delta t \hat{v}(x_j,t^{n}).$$
With the above procedure, we then consider various boundary velocity loading $\partial_t \hat{u}_i(x,t)$, $x\in\partial\Omega$, and force loading $g_i(x,t)$ scenarios, and solve for the corresponding HF displacement field $\hat{u}_i(x,t)$. Resultant data pairs $\{\hat{u}_i,g_i\}_{i=1}^N = \{\hat{u}_i(x_j,t^n),g_i(x_j,t^n): j=1,\ldots,J\}_{i=1,n=0}^{N,T/\Delta t}$ are employed as the training and validation datasets. Discretization parameters for the DNS solver are set to $\Delta t= 0.01$ and $\max{\Delta x}= 0.01$.

The homogenization problem is then to learn the kernel of the nonlocal operator $L_\phi$ that approximates the operator $L_{HF}$ from data $\{\hat{u},f\}$ generated by $L_{HF}[\hat{u}]=f$, where $f=\partial_{tt}\hat{u}-g$. Discretizing the time derivative in \eqref{eq: NLmodel} with the central difference scheme, we obtain
$$\dfrac{1}{\Delta t^2}(\hat{u}^{n+1}(x)-2\hat{u}^{n}(x)+\hat{u}^{n-1}(x))-g(x,t^n):=f^n(x),$$
where $\hat{u}^n(\cdot):=\hat{u}(\cdot,t^n)$ denotes the solution at time $t^n$. Given $\mathcal{D}=\{\hat{u}_i^n(x),f_i^n(x)\}_{i=1,n=1}^{N,T/\Delta t}$, our goal is to learn the kernel $\phi$. 
The loss functional is 
\begin{align}\label{eq: loss function meta}
    \calE(\phi) & = \frac{\Delta t}{NT}\sum_{k=1}^N\sum_{n=1}^{T/\Delta t}\| L_\phi[\hat{u}^n_k] -f^n_k\|^2_{L^2(\Omega)}. 
\end{align}

\subsection{Settings on real-world data}
In the learning problem, we consider four types of data and use the first three for training and the last one for validation of our algorithm. For all data we set $L = 0.2$, $\Delta t=0.02$, and the symmetric domain $\Omega = [-b, b]$. Two spatial resolutions, $\Delta=0.05$ and $\Delta=0.025$ are considered, which we denote as the ``coarse'' and ``fine'' datasets, respectively.
\begin{enumerate}\setlength{\itemindent}{5mm}
    \item[Type 1] \textit{Oscillating source (20 samples in total).} $b = 50$, $T=2$, $g(x,t) = \exp^{-(\frac{2x}{5jL})^2}\exp^{-(\frac{t-0.8}{0.8})^2}\cos^2(\frac{2\pi x}{jL})$, where $j=1,2,\cdots,20$. 
    \item[Type 2] \textit{Plane wave with $\cos$ loading (11 samples in total).} $b = 50$, $T=2$, $g(x,t) = 0$ and $\partial_t {u}(-50,t) = \cos(j t)$, where the loading frequency $j = 0.35,0.70,\cdots, 3.85$.
    \item[Type 3] \textit{Plane wave with $\sin$ loading (11 samples in total).} $b = 50$, $T=2$, $g(x,t) = 0$ and $\partial_t {u}(-50,t) = \sin(j t)$, where the loading frequency $j = 0.35,0.70,\cdots, 3.85$.
    \item[Type 4] \textit{Wave packet (3 samples in total).}   $b = 133.3$, $T=100$, $g(x,t) = 0$ and $\partial_t {u}(-b, t) = \sin(j t) \exp\left(-(t/5 -3)^2\right)$, for $j =1,\, 2, \,3$.
\end{enumerate}
Notice that the validation dataset (Type 4 dataset) is under a different loading condition from the training dataset, and with a much longer simulation time.

\begin{ack}
YY are supported by the National Science Foundation under award DMS 1753031, and the AFOSR grant FA9550-22-1-0197. YY would also like to like to thank Dr. Stewart Silling for his help on the DNS codes and for valuable discussions.  FL is grateful for supports from  NSF-1913243 and FA9550-20-1-0288. FL and QA would like to thank Quanjun Lang for helpful discussions on regularization.  


\end{ack}


\medskip

 \bibliographystyle{plain}
 {\small
\bibliography{ref_levy,ref_inverseP,ref_kernel_learning,ref_FeiLU2022_1,ref_prop,ref_regularization,ref_nonlocal_kernel,ref_sparseRegression,ref_operator}
}
\end{document}